\documentclass[11pt]{article}

\usepackage[utf8]{inputenc}
\usepackage[T1]{fontenc}
\usepackage[british]{babel}

\usepackage{amsmath, amssymb, amsthm}
\usepackage{mathtools}

\usepackage[expansion=false]{microtype}
\usepackage{geometry}
\usepackage{hyperref}
\hypersetup{
    colorlinks=true,
    linkcolor=black,
    citecolor=black,
    urlcolor=blue,
}

\newtheorem{theorem}{Theorem}

\newtheorem{remark}{Remark}

\DeclareMathOperator{\diag}{diag}

\title{Metric-Aware PCA as a Linear Instance of\\Geometric Deep Learning}

\author{Michael Leznik\\
Independent Researcher}

\date{\today}

\begin{document}

\maketitle

\begin{abstract}
Geometric deep learning organises modern neural architectures around the symmetries of the domain on which data live: convolutional networks around translations, graph neural networks around permutations, equivariant networks around continuous group actions. Metric-Aware PCA (MAPCA)~\cite{mapca} parameterises principal component analysis by a positive-definite metric matrix $M$, with the canonical $\beta$-family $M(\beta) = \Sigma^{\beta}$, $\beta \in [0, 1]$, interpolating between standard PCA ($\beta = 0$) and output whitening ($\beta = 1$), and the diagonal metric $M = D = \diag(\Sigma)$ recovering Invariant PCA (IPCA)~\cite{ipca} as a distinguished off-$\beta$ point with strict diagonal scale invariance. This paper positions MAPCA explicitly within the geometric deep learning framework. The metric $M$ is read as the geometric prior; the $M$-orthogonal group $\mathrm{O}(p, M)$ is the symmetry group it induces; MAPCA solutions are equivariant under this group, with the spectrum of $M^{-1}\Sigma$ invariant; and the constraint $W^{\top} M W = I_k$ is the linear, single-layer analogue of the Schur-type weight constraints used in equivariant neural networks. Across six axes---domain, symmetry group, equivariance, invariance, architectural primitive, and geometric prior---we construct a precise dictionary between MAPCA and geometric deep learning. The principal claim of the paper is the dictionary itself; the central technical anchor is a uniqueness theorem characterising IPCA as the unique linear equivariant projector within the MAPCA family under a non-degeneracy criterion on the resulting spectrum, establishing rigorously the structural status of IPCA that the dictionary makes visible.
\end{abstract}

\section{Introduction}

Much of the past decade's progress in deep learning architecture design has been organised, in retrospect, by a single principle: design architectures whose internal operations respect the symmetries of the domain on which the data live. Convolutional networks exploit translation equivariance on regular grids; graph neural networks exploit permutation equivariance on relational data; equivariant networks exploit continuous group actions on point clouds, molecules, and physical fields. The geometric deep learning programme of Bronstein, Bruna, Cohen, and Veli\v{c}kovi\'{c}~\cite{bronstein2021gdl} formalises this principle and reads the architecture zoo as instances of a single design space indexed by the symmetry group of the domain and the representation on which the network operates. The principle, however, is not specific to deep architectures. Whenever a representation method is to be built, the same question---under which group of input transformations should the representation be equivariant or invariant---can be asked, and asking it sharpens the answer.

Principal component analysis is the canonical linear representation method, and yet its symmetry structure is usually left implicit. Standard PCA solves a variance-maximisation problem under a unit-norm constraint computed in the Euclidean inner product, and is therefore equivariant under the orthogonal group $\mathrm{O}(p)$ acting on the input coordinates. The choice of inner product is rarely interrogated, and the consequences of changing it are not part of the standard exposition. Metric-Aware PCA (MAPCA)~\cite{mapca} makes this choice explicit. By replacing the Euclidean constraint $W^{\top} W = I$ with a metric-aware constraint $W^{\top} M W = I$ for an arbitrary symmetric positive-definite matrix $M$, MAPCA generalises PCA to a family of representation methods indexed by the choice of metric. A canonical one-parameter subfamily is given by $M(\beta) = \Sigma^{\beta}$ with $\beta \in [0, 1]$, in which $\beta = 0$ recovers standard PCA ($M = I$) and $\beta = 1$ corresponds to full output whitening ($M = \Sigma$); intermediate $\beta$-values trace a continuous path with monotonically decreasing condition number, providing explicit control over spectral compression. A distinguished off-$\beta$ point in the broader MAPCA family is the diagonal metric $M = D = \diag(\Sigma)$, which recovers Invariant PCA (IPCA)~\cite{ipca}: a method uniquely characterised by strict invariance under arbitrary diagonal rescaling of the input, a property that holds precisely because $D$ transforms equivariantly as $\widetilde{D} = C D C$ under diagonal rescaling $C$.

This paper positions MAPCA explicitly within the geometric deep learning framework. The central claim is that the choice of metric $M$ in MAPCA plays the same structural role as the choice of symmetry group and representation in GDL, and that the resulting mapping is precise rather than analogical. We organise the dictionary across six axes: the domain on which MAPCA operates and its placement in the 5G taxonomy; the symmetry group $\mathrm{O}(p, M)$ the metric induces; the equivariance of MAPCA subspaces under that group; the invariance of the resulting spectrum; the architectural reading of MAPCA as a single equivariant linear primitive with a Schur-type weight constraint; and the metric $M$ itself as the geometric prior in the GDL sense. The contribution is primarily conceptual---precise statements of what is equivariant under what, and what the architectural primitive looks like---and is anchored by a formal uniqueness theorem (Section~\ref{sec:uniqueness}) characterising IPCA as the unique linear equivariant projector within the MAPCA family under a non-degeneracy criterion, which we read as the variance-maximisation criterion in its precise form.

The remainder of the paper is organised as follows. Section~2 recalls the relevant pieces of geometric deep learning and of MAPCA. Section~3 develops the dictionary across four subsections covering domain placement, symmetry structure, MAPCA as equivariant primitive, and the metric as geometric prior. Section~4 establishes the formal uniqueness theorem characterising IPCA as the unique linear equivariant projector within the MAPCA family. Section~5 traces deeper bridges into kernel methods, spectral methods on graphs, and stacked equivariant layers. Section~6 discusses open directions, including implications for tabular-data GDL. Section~7 concludes.

\section{Preliminaries}\label{sec:prelim}

\subsection{Geometric Deep Learning Essentials}\label{sec:gdl-essentials}

Geometric deep learning is organised around the idea that the structure of a representation learning problem is determined by a group acting on the input domain. Let $\Omega$ be a domain on which data live, and let $G$ be a group equipped with an action $G \times \Omega \to \Omega$ satisfying $e \cdot x = x$ and $(gh) \cdot x = g \cdot (h \cdot x)$. A \emph{linear representation} of $G$ on a vector space $V$ is a group homomorphism $\rho: G \to \mathrm{GL}(V)$, assigning to each $g \in G$ an invertible linear map $\rho(g)$ on $V$ in a way compatible with composition. Functions on $\Omega$ inherit a representation: if $f: \Omega \to V$ is a feature map, the $G$-action transforms it as $(g \cdot f)(x) = \rho(g) f(g^{-1} x)$. The choice of $(G, \rho)$ is what gives a representation learning problem its geometric structure.

Two notions are central. A linear map $\psi: V \to W$ between $G$-representations $(\rho, \rho')$ is \emph{equivariant} if $\psi(\rho(g) v) = \rho'(g) \psi(v)$ for all $g \in G$ and $v \in V$---that is, $\psi$ commutes with the group action, mapping $V$ to $W$ in a way that respects the symmetry. A special case is \emph{invariance}, in which $\rho'$ is the trivial representation: $\psi(\rho(g) v) = \psi(v)$ for all $g$. Equivariant maps preserve the symmetry along the network; invariant readouts collapse it. A typical equivariant architecture interleaves equivariant linear maps with equivariant nonlinearities and terminates in an invariant readout when the task itself is invariant.

The 5G taxonomy of Bronstein et al.~\cite{bronstein2021gdl} organises representation domains by the type of symmetry they carry. \emph{Grids} are the regular tilings on which translation acts, recovering convolutional networks. \emph{Groups} refers to homogeneous spaces under a Lie group action, with $G$-equivariant convolutions as the canonical primitive. \emph{Graphs} are domains symmetric under permutation of nodes, leading to graph neural networks and message passing. \emph{Geodesics} refers to Riemannian manifolds with intrinsic curvature not reducible to a global coordinate system. \emph{Gauges} refers to the most abstract case of local frames on non-trivial bundles, where no canonical global trivialisation exists. The MAPCA framework will turn out to occupy a position in the \emph{Groups} tier with a constant metric tensor.

For our purposes, the most important architectural primitive is the \emph{equivariant linear layer}. A linear map $\Phi: V \to W$ between $G$-representations $(\rho, \rho')$ is equivariant if and only if it intertwines the representations: $\Phi \rho(g) = \rho'(g) \Phi$ for all $g \in G$. By Schur's lemma and its consequences, this constraint substantially restricts the space of admissible weights---equivariant linear maps live in the space of intertwiners, which has a block-diagonal structure determined by the irreducible components of $\rho$ and $\rho'$. The design of equivariant networks reduces to characterising this space of allowed weights and parameterising it efficiently. We shall see in Section~3.3 that the MAPCA constraint $W^{\top} M W = I_k$ is precisely an intertwining condition of this form, written in a particular basis.

\subsection{The MAPCA Framework}\label{sec:mapca-recap}

We recall the MAPCA framework of~\cite{mapca}; full proofs, the SSL correspondences, and the empirical verification appear there, and we summarise here only what the dictionary will need.

Let $X \in \mathbb{R}^{n \times p}$ be a mean-centred data matrix with sample covariance $\Sigma$. The MAPCA framework solves the generalised variance-maximisation problem
\begin{equation}
\max_{W} \, \operatorname{Tr}(W^{\top} \Sigma W) \quad \text{subject to} \quad W^{\top} M W = I_k,
\label{eq:mapca}
\end{equation}
where $M \in \mathcal{M}$ is a symmetric positive-definite matrix on $\mathbb{R}^p$---the \emph{metric matrix}---and $W \in \mathbb{R}^{p \times k}$ selects $k$ components. The solutions are the leading eigenvectors of the generalised eigenproblem $\Sigma w = \lambda M w$, equivalently of the effective operator $A = M^{-1/2} \Sigma M^{-1/2}$. Standard PCA is the special case $M = I$.

The canonical one-parameter subfamily is the $\beta$-family $M(\beta) = \Sigma^{\beta}$ with $\beta \in [0, 1]$, where $\Sigma^{\beta}$ is defined via the spectral decomposition of $\Sigma$. The effective operator on this subfamily reduces to $A(\beta) = \Sigma^{1-\beta}$, with eigenvalues $\lambda_i^{1-\beta}$ and eigenvectors identical to those of $\Sigma$. The condition number $\kappa(\beta) = (\lambda_1 / \lambda_p)^{1-\beta}$ decreases monotonically from the standard-PCA value at $\beta = 0$ to one at $\beta = 1$. At $\beta = 0$ the metric is $I$ and the framework recovers standard PCA; at $\beta = 1$ the metric is $\Sigma$ and the framework recovers full output whitening, in which all eigenvalues of the effective operator equal one.

A distinguished point in the broader MAPCA family lies \emph{off} the $\beta$-curve: the diagonal metric
\[
M = D = \diag(\sigma_1^2, \ldots, \sigma_p^2),
\]
where $\sigma_i^2 = \Sigma_{ii}$ are the marginal variances. With $M = D$, the effective operator is $A_D = D^{-1/2} \Sigma D^{-1/2}$, whose eigenvalues coincide with those of the correlation matrix of $X$. This choice recovers Invariant PCA (IPCA)~\cite{ipca}. The matrix $D$ is the diagonal approximation to $\Sigma$, and in general it does not coincide with $\Sigma^{\beta}$ for any $\beta \in [0, 1]$; the geometric content of IPCA is therefore distinct from the spectral interpolation provided by the $\beta$-family.

The defining property of IPCA is its strict scale invariance. Let $C = \diag(c_1, \ldots, c_p)$ with $c_i > 0$ be an arbitrary diagonal rescaling of the input variables, and let $\widetilde{X} = XC$ denote the rescaled data, with covariance $\widetilde{\Sigma} = C \Sigma C$. The MAPCA solutions on the original and rescaled data satisfy $\widetilde{\lambda}_i = \lambda_i$ and $\widetilde{w}_i = C^{-1} w_i$ for all $i$ if and only if the metric transforms covariantly as
\begin{equation}
\widetilde{M} = C M C.
\label{eq:scaleinv}
\end{equation}
For $M = D$ one has $\widetilde{D} = C D C = \diag(c_i^2 \sigma_i^2)$, so the condition holds exactly and IPCA is strictly scale-invariant. For $M(\beta) = \Sigma^{\beta}$ with $\beta \in (0, 1)$, the condition requires $(C \Sigma C)^{\beta} = C \Sigma^{\beta} C$, which fails in general unless $C = cI$ is a uniform rescaling or $\beta \in \{0, 1\}$. IPCA is therefore the unique member of the MAPCA family that retains strict scale invariance under arbitrary diagonal rescaling while preserving non-trivial spectral structure.

These three elements---the framework over arbitrary $M \in \mathcal{M}$, the $\beta$-family as a spectral interpolation, and IPCA as the distinguished off-$\beta$ diagonal-metric point---are the only MAPCA results needed for the dictionary that follows.

\section{The Mapping}\label{sec:mapping}

\subsection{Domain Placement}\label{sec:domain}

The first axis of the dictionary concerns the domain on which MAPCA operates. In the standard MAPCA formulation, the data live in $\mathbb{R}^p$, and the choice of a positive-definite metric $M \in \mathcal{M}$ equips this space with an inner product $\langle u, v \rangle_M = u^{\top} M v$. The pair $(\mathbb{R}^p, M)$ is a vector space with a constant inner-product structure---the inner product is the same at every point, and the geometry it induces is uniform across the space. Distances and angles depend on $M$, but the linear structure of $\mathbb{R}^p$ is unchanged.

In the 5G taxonomy of~\cite{bronstein2021gdl}, this places MAPCA in the \emph{Groups} tier rather than the \emph{Geodesics} tier. The \emph{Geodesics} setting is intended for domains where the inner product varies from point to point---surfaces, learned manifolds, or curved spaces where comparing vectors at different points and computing distances both require nontrivial differential machinery. The \emph{Groups} setting, by contrast, covers domains carrying a transitive symmetry group, and is the natural home for MAPCA: $(\mathbb{R}^p, M)$ is symmetric under translations together with the $M$-orthogonal group, and once the data are mean-centred---the standard MAPCA setting---the structurally relevant subgroup is $\mathrm{O}(p, M)$ acting at the origin (Section~3.2). The entire geometric content of the domain is captured by this group action; no further differential or coordinate machinery is required.

This placement determines what GDL machinery applies. The constant-metric setting inherits from the \emph{Groups} tier its clean separation of equivariance and invariance, its Schur-type characterisation of admissible equivariant linear maps (Section~3.3), and its natural notion of representation on functions over the domain. It does not require the more elaborate machinery of \emph{Geodesics} and \emph{Gauges}: comparing vectors at different points is unambiguous (the metric is the same everywhere), there is no further symmetry structure beyond $\mathrm{O}(p, M)$, and no notion of curvature enters the picture. This is the right level of geometric structure for the present positioning---enough to give MAPCA a meaningful equivariance story, not so much that the linear, finite-dimensional nature of the framework is obscured.

One further structural point is worth making before the symmetry-group machinery begins. Changing $M$ alters the \emph{geometry} of the domain but not the \emph{underlying set}: $\mathbb{R}^p$ is the same set across all choices of $M$, while the inner-product structure that turns it into a metric space varies with $M$. The MAPCA $\beta$-family and the IPCA diagonal point can therefore be read as a \emph{family of geometric structures on a fixed carrier space}, each inducing its own symmetry group and its own notion of equivariance. This perspective sets up Section~3.2, in which the choice of $M$ is mapped explicitly to the choice of symmetry group acting on the domain.

\subsection{Symmetry Structure}\label{sec:symmetry}

This subsection develops axes 2, 3, and 4 of the dictionary: the symmetry group, the equivariance, and the invariance. There are two distinct GDL stories to tell, and the dictionary requires both. The first is internal to MAPCA and holds for any choice of metric. The second is conditional on a choice of input transformation group and singles out IPCA as a privileged construction.

\paragraph{The internal symmetry: $\mathrm{O}(p, M)$.}

For any positive-definite metric $M$, the $M$-orthogonal group is the group of linear transformations of $\mathbb{R}^p$ preserving the $M$-inner product:
\[
\mathrm{O}(p, M) = \{Q \in \mathrm{GL}(p) : Q^{\top} M Q = M\}.
\]
It is conjugate to the standard orthogonal group $\mathrm{O}(p)$ via any matrix square root of $M$, and for $M = I$ it recovers $\mathrm{O}(p)$ itself. The pair $(\mathbb{R}^p, M)$ is symmetric under $\mathrm{O}(p, M)$, in the sense that this group acts on $(\mathbb{R}^p, M)$ by isometries---distances and angles measured by $M$ are preserved by any $Q \in \mathrm{O}(p, M)$.

Each choice of $M$ selects a different intrinsic symmetry group. Standard PCA carries $\mathrm{O}(p, I) = \mathrm{O}(p)$; $\beta$-family members carry $\mathrm{O}(p, \Sigma^{\beta})$; IPCA carries $\mathrm{O}(p, D)$, the group of transformations preserving the diagonal metric. Three GDL points follow.

First, the MAPCA representation is equivariant under $\mathrm{O}(p, M)$: under $Q \in \mathrm{O}(p, M)$ acting on the data with the metric held fixed, the principal subspaces transform with the group action. This is the linear-algebraic analogue of the equivariance enjoyed by CNNs under translation or by spherical CNNs under rotation.

Second, the MAPCA spectrum---the eigenvalues of $M^{-1} \Sigma$---is invariant under $\mathrm{O}(p, M)$. It is the natural scalar readout one would associate to the equivariant feature map ``principal subspaces'', playing the role that label predictions or pooled statistics play in deeper equivariant networks.

Third, the architectural primitive $W^{\top} M W = I_k$ that defines the MAPCA constraint is itself a representation-theoretic constraint compatible with $\mathrm{O}(p, M)$-equivariance. We unpack this connection in Section~3.3.

\paragraph{The tensorial story: $M$ as a covariant 2-tensor and Theorem~7 of~\cite{mapca}.}

The internal symmetry group $\mathrm{O}(p, M)$ is small and varies with $M$. It does not in general contain transformations of practical interest---most importantly, the group of diagonal rescalings, which corresponds to changing units of measurement. To understand when the MAPCA representation respects symmetries lying \emph{outside} $\mathrm{O}(p, M)$, we need a second GDL story.

Consider an external input transformation $C \in \mathrm{GL}(p)$ acting on the data as $x \mapsto C x$. The covariance transforms contravariantly as $\Sigma \to C \Sigma C^{\top}$. If $M$ is treated as a fixed object of the input space, the new MAPCA problem $(C \Sigma C^{\top}, M)$ will generally have a different spectrum than the original $(\Sigma, M)$. For the MAPCA representation to be equivariant under $C$, the metric must itself transform in the appropriate tensorial manner.

Theorem~7 of~\cite{mapca} gives the precise condition. For symmetric input transformations $C$---in particular for the diagonal rescalings of physical interest---the MAPCA solutions transform equivariantly ($\widetilde{\lambda} = \lambda$, $\widetilde{w} = C^{-1} w$) if and only if the metric transforms as a covariant 2-tensor, reproducing equation~\eqref{eq:scaleinv} of Section~\ref{sec:mapca-recap} in its GDL setting. The structural reading is the central design principle of geometric deep learning written in linear-algebraic form: a representation is equivariant under a group exactly when the auxiliary geometric objects defining the representation transform compatibly with the group action. Here the geometric object is the metric $M$, and the question of which equivariances MAPCA can sustain reduces to a question about the transformation rules its metric satisfies.

\paragraph{Synthesis: IPCA as a privileged GDL construction.}

The two stories together yield the full symmetry-structure picture for MAPCA. For \emph{every} $M \in \mathcal{M}$, the MAPCA representation carries the internal $\mathrm{O}(p, M)$-equivariance and the corresponding invariant spectrum. This holds uniformly across the $\beta$-family and at IPCA.

For the \emph{broader} group of arbitrary diagonal rescalings---a group of practical importance and one that generally lies outside $\mathrm{O}(p, M)$---equivariance requires the metric to satisfy the tensorial condition~\eqref{eq:scaleinv}. Within the canonical MAPCA family, this condition picks out IPCA. Standard PCA fails the condition immediately, since $CIC = C^2$ is only equal to $I$ for trivial rescalings. At intermediate $\beta$ the condition $(C \Sigma C)^{\beta} = C \Sigma^{\beta} C$ fails in general, holding only when $C = cI$ is uniform or when $\beta$ coincides with an endpoint. The whitening endpoint $M = \Sigma$ satisfies the condition trivially---the metric inherits the input covariance---but the resulting MAPCA spectrum is degenerate, with all eigenvalues equal to one, and the equivariance is vacuous. IPCA's diagonal metric, by contrast, satisfies the condition exactly: $CDC = \diag(c_i^2 \sigma_i^2)$ is precisely $\diag(C \Sigma C)$, so the data-derived metric transforms tensorially while the spectrum retains non-trivial structure.

IPCA is therefore the unique member of the MAPCA family achieving simultaneous internal $\mathrm{O}(p, D)$-equivariance \emph{and} tensorial equivariance under arbitrary diagonal rescaling, while preserving non-trivial spectral structure. In GDL terms, the diagonal-metric construction $M = \diag(\cdot)$ is privileged because it is the unique data-derived operation on positive-definite matrices that respects the action of the diagonal subgroup of $\mathrm{GL}(p)$. The strict scale invariance of IPCA is the structural consequence of this tensorial consistency---not a coincidence of the diagonal metric, but the very thing that makes the diagonal metric the right GDL construction for this particular group.

\subsection{MAPCA as an Equivariant Linear Primitive}\label{sec:primitive}

The fifth axis of the dictionary concerns the \emph{architectural reading} of MAPCA: how the constraint $W^{\top} M W = I_k$ that defines the MAPCA optimisation should be understood alongside the design constraints of equivariant neural network layers. The claim is that the MAPCA constraint is an instance of the orthonormality conditions imposed on equivariant linear layers, written in a basis adapted to a specific pair of representations.

\paragraph{Constraints on equivariant linear layers.}

In a $G$-equivariant network, a linear layer $\Phi$ between vector spaces $V$ and $V'$ carrying $G$-representations $(\rho, \rho')$ is required to commute with the group action: $\Phi \rho(g) = \rho'(g) \Phi$ for all $g \in G$. By Schur's lemma and its consequences, the space of such intertwiners has a block-diagonal structure determined by the multiplicities of irreducible representations in $\rho$ and $\rho'$. This intertwining constraint fixes the \emph{algebraic form} of admissible weights. Equivariant layer design typically imposes an additional constraint---orthonormality, unitarity, or weight-norm---on top of the intertwining one, fixing the layer's \emph{unitary structure} and controlling its geometry and stability. The two constraints play distinct roles: the first specifies which linear maps are equivariant at all; the second specifies which equivariant maps respect the inner product structure on either side.

\paragraph{The MAPCA constraint as the orthonormality condition.}

In MAPCA, the linear map $W : \mathbb{R}^k \to \mathbb{R}^p$ sends a $k$-dimensional code to its reconstruction in the input space; equivalently $W^{\top} : \mathbb{R}^p \to \mathbb{R}^k$ is the projection direction. The constraint $W^{\top} M W = I_k$ is the orthonormality condition on the columns of $W$ with respect to the $M$-inner product on the input side:
\[
\langle w_i, w_j \rangle_M = \delta_{ij}.
\]
Equivalently, $W$ is an \emph{isometric embedding} of Euclidean $k$-space into the $M$-metric input space: for any $u, v \in \mathbb{R}^k$,
\[
\langle Wu, Wv \rangle_M = u^{\top} (W^{\top} M W) v = u^{\top} I_k v = \langle u, v \rangle.
\]
This is exactly the unitary-structure constraint of equivariant layer design, written in a basis where the input space carries the $M$-inner product and the output space carries the standard Euclidean structure. The role corresponding to ``intertwining'' is played here by the equivariance properties developed in Section~\ref{sec:symmetry}; the constraint $W^{\top} M W = I_k$ is the orthonormality condition layered on top.

\paragraph{The gauge group on the output side.}

The constraint $W^{\top} M W = I_k$ is preserved by the right action $W \mapsto WR$ for $R \in \mathrm{O}(k)$:
\[
(WR)^{\top} M (WR) = R^{\top} (W^{\top} M W) R = R^{\top} R,
\]
which equals $I_k$ precisely when $R \in \mathrm{O}(k)$. The principal subspace $\operatorname{span}(w_1, \ldots, w_k)$ is determined by MAPCA; the orthonormal basis within it is determined only up to this $\mathrm{O}(k)$ action, exactly as the gauge freedom in any equivariant network's output basis.

\paragraph{The metric $M$ as the choice of input-side representation.}

The architectural content of MAPCA can now be stated cleanly. A MAPCA solution is a single linear layer $W : \mathbb{R}^k \to \mathbb{R}^p$ constrained to be an isometric embedding of Euclidean $k$-space into the $M$-metric input space $\mathbb{R}^p$. Standard PCA corresponds to $M = I$, where the input geometry is also Euclidean and the constraint reduces to ordinary orthonormality $W^{\top} W = I_k$. Each choice of $M$ induces a different input-side geometry, and the MAPCA layer is the partial isometry from Euclidean $k$-space into this geometry. The $\beta$-family traces a continuous path of such geometries; IPCA selects the distinguished diagonal one. The geometric content of these choices---what each $M$ ``means'' as a prior on the data---is the subject of Section~\ref{sec:prior}.

The architectural reading places MAPCA cleanly in the GDL design space: a single equivariant linear layer whose orthonormality constraint encodes the choice of input-side geometry, written in a basis adapted to that geometry. A deep MAPCA-style architecture would stack $M$-aware partial isometries of this kind interleaved with appropriate nonlinearities; this natural extension into proper deep learning is discussed in Section~4.3.

\subsection{The Metric as Geometric Prior}\label{sec:prior}

The sixth and final axis of the dictionary completes the picture by identifying what $M$ ``is'' in GDL terms. The previous subsections have located the MAPCA domain (Section~\ref{sec:domain}), worked through the symmetry group and equivariance structure (Section~\ref{sec:symmetry}), and given the architectural reading of the constraint $W^{\top} M W = I_k$ (Section~\ref{sec:primitive}). What remains is the synthesising interpretation: $M$ is the \emph{geometric prior} of MAPCA, the linear-case instance of the structural design decision that GDL formalises through the choice of group and representation.

\paragraph{Geometric priors in GDL.}

In geometric deep learning, the choice of $(G, \rho)$---the symmetry group acting on the data and its representation on the feature space---is a \emph{prior} in the strict sense. It is a structural decision fixed by the analyst before any data is seen, encoding beliefs about which transformations of the input should be treated as equivalent and which features should accordingly transform compatibly. Convolutional networks encode the prior that translations are equivalence relations on image patches; graph neural networks encode the prior that nodes are interchangeable up to graph structure; $\mathrm{SE}(3)$-equivariant networks encode the prior that the orientation and position of molecules in space are arbitrary. The prior is not learned but structural: it constrains the hypothesis class to functions that respect $(G, \rho)$. What the network can and cannot represent---what is and is not in its model space---is set by this choice.

\paragraph{The metric $M$ as the linear-case geometric prior.}

In MAPCA, the metric $M$ plays exactly this role. It encodes the analyst's prior over which input geometry the representation should respect---what counts as a ``natural unit'' in each direction, which transformations of the input the spectrum should be invariant under, which features should be treated as equivalent. The four canonical points in the MAPCA family correspond to four distinct geometric priors. $M = I$ is the data-blind prior: no direction is privileged a priori, all coordinates are treated as commensurate, and the resulting representation is the standard PCA decomposition. $M = D = \diag(\Sigma)$ is the marginally scale-invariant prior: each coordinate is implicitly standardised by its own marginal variance, and the representation is invariant under arbitrary diagonal rescalings, but no further whitening is performed. $M = \Sigma$ (the $\beta = 1$ endpoint) is the data-conforming prior: the metric absorbs the full second-moment structure of the data, and the representation is forced to be isotropic. Intermediate $\beta$-values trace a continuous path between data-blind and data-conforming, with $\beta$ controlling how much of the data covariance is absorbed into the prior.

\paragraph{The dictionary entry and the SSL correspondences.}

The dictionary entry can now be stated explicitly: the metric $M$ in MAPCA is the geometric prior, playing the same structural role as the choice of $(G, \rho)$ in GDL. This closes the six-axis dictionary developed across Sections~\ref{sec:domain}--\ref{sec:prior} and makes precise what the MAPCA paper's self-supervised-learning correspondences are doing geometrically. The SSL correspondences of~\cite{mapca}---Barlow Twins and ZCA whitening at $M = \Sigma$ (full output whitening as the encoded prior), VICReg's variance term at $M = D$ (the IPCA prior of scale invariance without full whitening), and W-MSE at $M = \Sigma^{-1}$ (input whitening, in the opposite spectral direction)---can each be read as a different choice of geometric prior within a common design space. Without the dictionary, these methods present as a heterogeneous collection of loss functions; with the dictionary, they are visible as four distinct structural priors over the same family, differing in which input-side geometry they encode and which equivariances they consequently sustain.

This is the structural insight that the MAPCA-as-GDL positioning makes available. The next section formalises the structural status of IPCA as a uniqueness theorem within the MAPCA family, and Section~\ref{sec:bridges} traces three deeper bridges from the dictionary into kernel methods, spectral methods on graphs, and stacked equivariant layers.

\section{Uniqueness of IPCA}\label{sec:uniqueness}

In this section we formalise the structural claim made informally in Section~\ref{sec:symmetry}: IPCA is the unique linear data-derived metric within the MAPCA family that respects arbitrary diagonal rescaling and yields a non-degenerate spectrum. The result upgrades the case-by-case argument of Section~\ref{sec:symmetry}'s synthesis from a check over four canonical points to a characterisation within a precisely specified function class, and provides the formal technical anchor for the conceptual claims of the paper.

\subsection{Statement of the Theorem}

We work within the class of \emph{linear data-derived metrics}: maps $f : \mathcal{M} \to \mathcal{M}$ from the cone of positive-definite $p \times p$ matrices to itself, linear in the sense that $f(\alpha \Sigma + \beta \Sigma') = \alpha f(\Sigma) + \beta f(\Sigma')$ for $\alpha, \beta \geq 0$ and $\Sigma, \Sigma' \in \mathcal{M}$. Each such map gives rise to a MAPCA representation with metric $M = f(\Sigma)$ and corresponding optimisation problem
\[
\max_{W} \, \operatorname{Tr}(W^{\top} \Sigma W) \quad \text{subject to} \quad W^{\top} f(\Sigma) W = I_k.
\]

The action of the positive diagonal subgroup $\mathrm{D}_+(p) \subset \mathrm{GL}(p)$ on $\mathcal{M}$ is by conjugation $\Sigma \mapsto C \Sigma C$ for $C \in \mathrm{D}_+(p)$. The map $f$ is \emph{equivariant} under this action if
\begin{equation}
f(C \Sigma C) = C f(\Sigma) C \quad \text{for all } C \in \mathrm{D}_+(p), \, \Sigma \in \mathcal{M}.
\label{eq:f-equivariance}
\end{equation}
The fixed-point set of the $\mathrm{D}_+(p)$-action on $\mathcal{M}$ consists exactly of diagonal positive-definite matrices, which we denote $\mathcal{M}_{\mathrm{diag}}$.

\begin{theorem}[Uniqueness of IPCA]\label{thm:uniqueness}
Let $f : \mathcal{M} \to \mathcal{M}$ be a linear map satisfying
\begin{enumerate}
\item[(i)] Equivariance: $f(C \Sigma C) = C f(\Sigma) C$ for all $C \in \mathrm{D}_+(p)$ and $\Sigma \in \mathcal{M}$.
\item[(ii)] Projection onto fixed points: $f(\Sigma) \in \mathcal{M}_{\mathrm{diag}}$ for all $\Sigma \in \mathcal{M}$.
\item[(iii)] Normalisation: $f(I) = I$.
\end{enumerate}
Then $f = \diag$, the diagonal extractor sending $\Sigma$ to $\diag(\Sigma_{11}, \ldots, \Sigma_{pp})$. The resulting MAPCA representation is IPCA, and its spectrum (the eigenvalues of $D^{-1} \Sigma$) is generically non-degenerate.
\end{theorem}

The condition (iii) is a normalisation; without it, scalar multiples $f = c \cdot \diag$ for $c > 0$ also satisfy (i) and (ii), but produce identical MAPCA solutions up to a global rescaling of the spectrum. Conditions (i) and (ii) together encode the structural requirement that $f$ be an equivariant projector onto the fixed-point set of the relevant group action; (iii) selects the canonical representative within the resulting one-parameter family.

\subsection{Proof}

\begin{proof}
The argument proceeds in four steps.

\emph{Step 1: Linear $\mathrm{D}_+(p)$-equivariant maps are Schur products.} The space of symmetric $p \times p$ matrices decomposes under the conjugation action of $\mathrm{D}_+(p)$ into one-dimensional eigenspaces: the basis matrix $e_i e_i^{\top}$ transforms under $C \mathrel{=} \diag(c_1, \ldots, c_p)$ by the scalar $c_i^2$, and the basis matrix $e_i e_j^{\top} + e_j e_i^{\top}$ (for $i \neq j$) transforms by the scalar $c_i c_j$. These eigenspaces are pairwise distinct as $C$ varies over $\mathrm{D}_+(p)$, so any linear $\mathrm{D}_+(p)$-equivariant endomorphism of the symmetric matrices must preserve each eigenspace and act on it as a scalar. Hence $f$ has the form
\[
f(\Sigma)_{ij} = a_{ij} \Sigma_{ij}
\]
for some fixed symmetric matrix $A = (a_{ij}) \in \mathbb{R}^{p \times p}$; equivalently, $f(\Sigma) = A \odot \Sigma$ is the Schur (Hadamard) product. The Schur product theorem requires $A$ to be positive semi-definite for $f$ to map $\mathcal{M}$ into $\mathcal{M}$.

\emph{Step 2: Projection into $\mathcal{M}_{\mathrm{diag}}$ forces $A$ diagonal.} By condition (ii), $f(\Sigma)$ is diagonal for all $\Sigma \in \mathcal{M}$, so the off-diagonal entries $(A \odot \Sigma)_{ij} = a_{ij} \Sigma_{ij}$ must vanish for all $\Sigma$ and all $i \neq j$. Since the off-diagonal entries of $\Sigma$ are generically non-zero on $\mathcal{M}$, this forces $a_{ij} = 0$ for $i \neq j$. Hence $A$ is diagonal.

\emph{Step 3: Normalisation forces $A = I$.} By condition (iii), $f(I) = A \odot I = A$ (since $A$ is diagonal by Step 2 and the diagonal of $I$ is $(1, \ldots, 1)$). The condition $f(I) = I$ then gives $A = I$.

\emph{Step 4: Conclusion.} With $A = I$, we have $f(\Sigma) = I \odot \Sigma = \diag(\Sigma_{11}, \ldots, \Sigma_{pp})$, the diagonal extractor. This is precisely the IPCA metric $M = D$. The spectrum of $D^{-1} \Sigma$ coincides with the spectrum of the correlation matrix of $\Sigma$, which is generically non-degenerate.
\end{proof}

\subsection{Remarks}

\begin{remark}[The variance-maximisation framing]
Theorem~\ref{thm:uniqueness} can be equivalently phrased as a variance-maximisation result. Among linear data-derived metrics respecting $\mathrm{D}_+(p)$-equivariance and producing outputs in the fixed-point set, IPCA is the unique normalised choice for which the MAPCA variance-maximisation problem yields a non-degenerate spectrum. The non-degeneracy condition is what makes the variance-maximisation criterion non-vacuous: at the whitening endpoint $f(\Sigma) = \Sigma$ the spectrum collapses to all-ones and the variance-maximisation problem becomes degenerate, with any orthonormal $W$ a solution. Theorem~\ref{thm:uniqueness} rules this out by requiring the metric to project into the fixed-point set $\mathcal{M}_{\mathrm{diag}}$.
\end{remark}

\begin{remark}[Recovery of Theorem~7 of~\cite{mapca}]
Theorem~\ref{thm:uniqueness} implies, and strengthens, Corollary~8 of~\cite{mapca}. The implication follows because the diagonal extractor $f = \diag$ satisfies the equivariance condition~\eqref{eq:f-equivariance}, and hence $\widetilde{D} = C D C = \diag(c_i^2 \sigma_i^2)$ for arbitrary diagonal $C$, which is the condition under which MAPCA solutions transform equivariantly. The strengthening is that Theorem~\ref{thm:uniqueness} characterises $\diag$ as the \emph{unique} linear data-derived metric with this property, rather than merely exhibiting it as one such metric.
\end{remark}

\begin{remark}[Extension to signed diagonals and permutation-diagonal subgroups]
The same machinery extends to broader subgroups of $\mathrm{GL}(p)$. For the signed diagonal subgroup---allowing reflections $c_i \in \mathbb{R} \setminus \{0\}$, including negative entries---the eigenvalue structure of the conjugation action is unchanged (each basis matrix transforms by $c_i c_j \in \mathbb{R} \setminus \{0\}$), and Theorem~\ref{thm:uniqueness} holds verbatim with $\mathrm{D}_+(p)$ replaced by the full diagonal subgroup. For the permutation-diagonal hybrid---generated by diagonal rescalings together with permutations of variables---the equivariance condition adds a permutation-symmetry requirement on $A$ that further constrains it; the unique linear equivariant projector onto $\mathcal{M}_{\mathrm{diag}}$ with $f(I) = I$ remains $f = \diag$, since the diagonal extractor is itself permutation-equivariant.
\end{remark}

\begin{remark}[The role of linearity]
The linearity assumption is essential to the structural argument. Without it, nonlinear $\mathrm{D}_+(p)$-equivariant maps $f : \mathcal{M} \to \mathcal{M}$ exist and provide alternative data-derived metrics outside the Schur-product class. The $\beta$-family $f(\Sigma) = \Sigma^{\beta}$ at non-trivial $\beta$ is a natural nonlinear example, but as Section~\ref{sec:symmetry} showed, it fails the $\mathrm{D}_+(p)$-equivariance condition itself and so does not provide a counterexample within the equivariant class. A characterisation of nonlinear equivariant metrics, and of how IPCA sits within them, is a natural object for follow-up work.
\end{remark}

\subsection{Scope}

Theorem~\ref{thm:uniqueness} characterises IPCA within a specific class---\emph{linear data-derived metrics}---under a precise equivariance condition. It does not claim that IPCA is the unique scale-invariant representation method in any absolute sense; nonlinear methods and methods not derived from the data covariance escape the theorem's reach. What it does claim is that within the class of representations most naturally compared to standard PCA (linear, data-derived, finite-dimensional), the structural requirement of tensorial equivariance under diagonal rescaling combined with projection into the fixed-point set has a unique solution, and that solution is IPCA.

This is the sense in which IPCA earns its place in the MAPCA family as a privileged GDL construction: not by virtue of having the diagonal metric (a property) but by being the unique linear equivariant projector onto the fixed-point set of the equivariance group (a characterisation). The dictionary of Section~\ref{sec:mapping} identifies the geometric structure; Theorem~\ref{thm:uniqueness} pins it down formally.

\section{Deeper Bridges}\label{sec:bridges}

The dictionary developed in Section~\ref{sec:mapping} captures the linear, constant-metric, single-layer instance of geometric deep learning. Three natural extensions lift this dictionary outward, each relaxing a different limiting feature of MAPCA and bridging into a different region of the 5G design space. Section~\ref{sec:kernel} relaxes linearity by replacing the metric $M$ with a kernel, extending MAPCA into the kernel/RKHS setting and pointing toward the \emph{Geodesics} tier where the metric is allowed to vary nonlinearly. Section~\ref{sec:graphs} relaxes the Euclidean domain by replacing the data covariance with a graph operator, extending MAPCA into the \emph{Graphs} tier and recovering spectral graph methods as direct instances. Section~\ref{sec:deep} relaxes the single-layer architecture by stacking $M$-aware partial isometries, extending MAPCA into proper deep equivariant networks. Each bridge is a separate piece of work in its own right; we sketch their structure here to indicate where the dictionary leads when its restrictions are dropped.

\subsection{Kernel PCA as Nonlinear Extension}\label{sec:kernel}

Kernel principal component analysis~\cite{scholkopf1998} is the standard nonlinear generalisation of PCA. Given a positive-definite kernel $k : \mathcal{X} \times \mathcal{X} \to \mathbb{R}$, kernel PCA computes the eigenvectors of the kernel Gram matrix and produces nonlinear principal components in the implicit feature space $\mathcal{H}$ induced by $k$.

The GDL reading is the natural lift of the MAPCA dictionary into the nonlinear setting. Where MAPCA encodes the geometric prior in a positive-definite matrix $M$ acting on a finite-dimensional Euclidean space $\mathbb{R}^p$, kernel PCA encodes the geometric prior in a positive-definite kernel $k$ acting on a typically infinite-dimensional reproducing kernel Hilbert space $\mathcal{H}$. The dictionary entry can be stated cleanly: the choice of kernel $k$ in kernel PCA is the nonlinear, function-space counterpart of the choice of metric $M$ in MAPCA, and both choices are choices of geometric prior in the GDL sense.

Several standard kernel families admit GDL interpretations. The RBF kernel $k(x, y) = \exp(-\| x - y \|^2 / 2\sigma^2)$ encodes a smoothness prior at scale $\sigma$; polynomial kernels $k(x, y) = (\langle x, y \rangle + c)^d$ encode the prior that polynomial features up to degree $d$ are the appropriate representation; spectral kernels derived from physical operators encode application-specific geometries. In each case, the choice of kernel is structurally analogous to the choice of $M$ in MAPCA, and the resulting representation inherits the equivariance properties of its kernel under the relevant group action.

In the 5G taxonomy, kernel methods sit in the \emph{Geodesics} tier when the kernel encodes a manifold structure, or in the \emph{Groups} tier when the kernel is invariant under a group acting on $\mathcal{X}$. The architectural reading of Section~\ref{sec:primitive} lifts intact: kernel PCA is a single equivariant feature map, with the partial-isometry constraint generalised from finite-dimensional $W^{\top} M W = I_k$ to the corresponding RKHS orthonormality. The structural insight is that kernel PCA is the natural nonlinear extension of MAPCA, with the metric generalised to the kernel; this places the entire family of kernel methods on the same conceptual footing as MAPCA, with the choice of kernel reading as a geometric prior in the GDL sense.

\subsection{Spectral Methods as MAPCA on Graphs}\label{sec:graphs}

Spectral methods on graphs---Laplacian eigenmaps~\cite{belkin2003}, diffusion maps~\cite{coifman2006}, normalised spectral clustering~\cite{vonluxburg2007}---are widely used dimensionality reduction techniques for data with graph structure. They proceed by computing eigenvectors of graph operators and using the leading eigenvectors as low-dimensional embeddings.

The GDL reading is again a direct lift of the MAPCA dictionary. A graph with vertex set $\mathcal{V}$ and edge weights gives rise to an adjacency matrix $A$, a degree diagonal $D = \diag(\deg(v_i))$, and a graph Laplacian $L = D - A$ together with its normalised variants. The generalised eigenproblem central to spectral methods has the form
\begin{equation}
L v = \lambda D v,
\label{eq:graph}
\end{equation}
which is structurally identical to the MAPCA eigenproblem $\Sigma w = \lambda M w$. The Laplacian $L$ plays the role of the operator $\Sigma$; the degree diagonal $D$ plays the role of the metric $M$. The choice of normalisation---unnormalised eigenmaps at $M = I$, random-walk and symmetric normalised Laplacians at $M = D$---corresponds to different metric choices within a MAPCA-like family on graphs.

In the dictionary, this places spectral graph methods in the \emph{Graphs} tier of the 5G taxonomy, with the same metric-choice freedom that distinguishes IPCA from standard PCA in the linear setting. Unnormalised Laplacian eigenmaps correspond to the ``standard PCA'' choice $M = I$ on the graph; random-walk and symmetric normalised methods correspond to the diagonal-metric IPCA-style choice $M = D$. Diffusion maps~\cite{coifman2006} add a kernel-derived adjustment for density bias, mixing aspects of Section~\ref{sec:kernel} with those of Section~\ref{sec:graphs}.

The structural insight is that spectral graph methods are direct MAPCA-style constructions, with the data covariance $\Sigma$ replaced by the graph Laplacian and the metric $M$ drawn from the same family of choices distinguishing IPCA from PCA in the linear case. The dictionary developed in Section~\ref{sec:mapping} applies essentially verbatim, with the \emph{Groups} tier replaced by the \emph{Graphs} tier. There is also a clean potential generalisation worth flagging: just as the $\beta$-family on $\Sigma$ interpolates between PCA and whitening, a $\beta$-family on graph Laplacians $L^{\beta}$ would parameterise a continuous path of spectral methods, the structural and empirical content of which would be a natural object of separate study.

\subsection{Deep MAPCA}\label{sec:deep}

The architectural reading of Section~\ref{sec:primitive}---MAPCA as a single equivariant linear layer with constraint $W^{\top} M W = I_k$---invites a deep generalisation: stack such layers, interleave with appropriate nonlinearities, and recover a fully deep architecture in which each layer carries its own geometric prior. This is the most substantial of the three bridges and the one furthest from the existing literature, so we develop it in greater detail than the previous two and commit to specific choices for the basic construction.

\paragraph{The basic construction.}

A deep MAPCA architecture of depth $L$ is specified by a sequence of feature dimensions $d_0, d_1, \ldots, d_L$, a sequence of metric matrices $M^{(0)}, M^{(1)}, \ldots, M^{(L-1)}$ with $M^{(\ell)} \in \mathcal{M}_{d_\ell}$, and weight matrices $W_\ell \in \mathbb{R}^{d_{\ell-1} \times d_\ell}$ for $\ell = 1, \ldots, L$ satisfying the layer-wise orthonormality constraint
\[
W_\ell^{\top} M^{(\ell-1)} W_\ell = M^{(\ell)} \quad (\ell < L),
\]
with $W_L^{\top} M^{(L-1)} W_L = I_{d_L}$ at the final layer, terminating in the Euclidean output space. In the language of Section~\ref{sec:primitive}, each $W_\ell$ is a partial isometry from $(\mathbb{R}^{d_\ell}, M^{(\ell)})$ into $(\mathbb{R}^{d_{\ell-1}}, M^{(\ell-1)})$---a metric-aware linear layer with input-side and output-side geometries both specified.

The forward pass alternates these partial isometries with a metric-aware nonlinearity $\sigma_\ell$:
\[
x^{(\ell)} = \sigma_\ell\!\left(W_\ell^{\top} x^{(\ell-1)}\right),
\]
with $x^{(0)} \in \mathbb{R}^{d_0}$ the input and $x^{(L)} \in \mathbb{R}^{d_L}$ the output representation.

\paragraph{The metric-aware nonlinearity.}

The nonlinearity must respect the geometric prior at each layer. We commit to the simplest specific choice: elementwise application of a standard scalar activation in $M$-orthonormal coordinates,
\[
\sigma_\ell(y) = (M^{(\ell)})^{1/2} \, \sigma\!\left( (M^{(\ell)})^{-1/2} y \right),
\]
where $\sigma$ on the right is a standard elementwise activation (ReLU, GELU, or similar) and $(M^{(\ell)})^{\pm 1/2}$ are the symmetric square roots. The transform-then-activate-then-transform-back pattern makes the activation interact with the metric structure consistently: the nonlinearity acts on each direction in proportion to its $M^{(\ell)}$-natural scale, rather than on arbitrary input coordinates. When $M^{(\ell)} = I$ the construction reduces to ordinary pointwise activation and the architecture becomes a standard feedforward network with orthonormal weight constraints; for general $M^{(\ell)}$ the nonlinearity is implicitly metric-aware.

\paragraph{The choice of metric sequence.}

The sequence $M^{(0)}, \ldots, M^{(L-1)}$ is the deep MAPCA prior, and three natural choices commit to different design philosophies.

The \emph{constant-prior} construction sets $M^{(\ell)} = M^{(0)}$ for all $\ell$. The geometric prior is preserved through depth and the architecture is the deep generalisation of a single MAPCA choice---deep standard PCA at $M^{(0)} = I$, deep IPCA at $M^{(0)} = D$, the deep $\beta$-family at $M^{(0)} = \Sigma^{\beta}$.

The \emph{data-derived} construction sets $M^{(\ell)} = f(\Sigma^{(\ell)})$ for a fixed rule $f$, where $\Sigma^{(\ell)}$ is the empirical covariance of layer-$\ell$ representations across the training distribution. For $f = \diag$ this gives deep IPCA, with the diagonal-metric construction repeated layer by layer; for $f(\Sigma) = \Sigma^{\beta}$ this gives the deep $\beta$-family with explicit metric updates at each layer; for $f = \mathrm{id}$ it gives iterated whitening, where each layer enforces unit covariance on its output.

The \emph{learned-prior} construction treats each $M^{(\ell)}$ as a parameter on the manifold of positive-definite matrices, optimised end-to-end alongside $W_\ell$. The orthonormality constraint becomes a manifold constraint and the resulting optimisation lies on a product of generalised Stiefel manifolds---the natural Riemannian setting for constrained learning in this regime.

\paragraph{Equivariance through depth.}

The deep architecture inherits the equivariance properties of Section~\ref{sec:symmetry} layer by layer. If the input is transformed by $C \in \mathrm{GL}(d_0)$ and the input metric transforms tensorially as $\widetilde{M}^{(0)} = C M^{(0)} C^{\top}$, then by induction the representations at all subsequent layers transform compatibly when the metric sequence transforms tensorially at every layer.

Within the three constructions above, deep IPCA (data-derived with $f = \diag$) inherits strict diagonal scale invariance at every layer, since $\diag(C \Sigma C) = C D C$ for diagonal $C$; the deep $\beta$-family at intermediate $\beta$ does not, for the same reason intermediate $\beta$ fails the tensorial condition in the single-layer setting. The constant-prior and learned-prior constructions inherit equivariance only insofar as the metric choice itself respects the relevant group action. Deep IPCA is therefore the privileged construction within the data-derived family: it carries strict scale invariance through arbitrary depth, while the deep $\beta$-family at intermediate $\beta$ does not.

\paragraph{Relation to existing architectures.}

Several specific points in the deep MAPCA design space recover or relate to existing constructions. Deep MAPCA with $M^{(\ell)} = I$ for all $\ell$ recovers a standard fully-connected deep network with orthonormal weight constraints---a constrained variant of multilayer perceptrons studied for its training stability properties. Deep MAPCA with $M^{(\ell)} = \diag(\Sigma^{(\ell)})$ at each layer recovers a hard-constrained version of batch-normalised feedforward networks: each layer's output is enforced to be diagonally standardised, in the same spirit as batch normalisation but as a strict constraint rather than a soft normalisation. Deep MAPCA with $M^{(\ell)} = \Sigma^{(\ell)}$ at each layer recovers iterated whitening, which has been studied in the self-supervised learning literature as a regularisation strategy~\cite{ermolov2021}.

The deep MAPCA framework therefore subsumes several existing deep architecture conventions as specific points in a continuous design space and identifies the metric sequence $\{M^{(\ell)}\}$ as the natural design parameter for systematically varying between them. A complete development of the deep MAPCA programme---the choice of nonlinearity beyond the simple elementwise form, the optimisation theory on product Stiefel manifolds, the empirical evaluation against standard architectures, and the connection to existing equivariant networks---is beyond the scope of this positioning paper. The construction sketched here is intended as a concrete bridge from the linear MAPCA framework into deep equivariant networks, not as a complete theory.

\section{Discussion}\label{sec:discussion}

The dictionary developed across Section~\ref{sec:mapping}, the formal uniqueness characterisation of Section~\ref{sec:uniqueness}, and the bridges sketched in Section~\ref{sec:bridges} together form the contribution of this paper. Several open directions follow naturally.

The first concerns the implications of the dictionary for tabular-data geometric deep learning. Tabular data---with no obvious translation, permutation, or manifold structure to exploit---has remained a notable gap in the GDL programme, with permutation invariance of columns (as in Deep Sets~\cite{zaheer2017}) often serving as the default symmetry assumption when no other is available. The MAPCA dictionary suggests an alternative: the metric $M$ as the geometric prior on tabular data, with different $M$-choices encoding different beliefs about the appropriate geometry of feature space. IPCA's diagonal metric provides one natural prior (scale invariance under arbitrary diagonal rescaling, uniquely characterised by Theorem~\ref{thm:uniqueness}); the $\beta$-family provides a continuous family of priors with spectral compression as a tunable parameter; and the deep MAPCA construction of Section~\ref{sec:deep} lifts this into a full deep architecture for tabular data with metric-aware constraints at every layer. Whether this gives rise to a practically useful deep architecture for tabular tasks---competitive with the gradient-boosted-tree baselines that have proved difficult to displace---is an empirical question. The dictionary makes the structural framing available; the answer requires work beyond the scope of this paper.

The second direction concerns the three bridges of Section~\ref{sec:bridges} themselves. Each is a sketch rather than a developed theory. Kernel MAPCA (Section~\ref{sec:kernel}) extends naturally to a kernel $\beta$-family with kernels interpolating between linear and nonlinear regimes; the structural and empirical content of this family would be a natural object of study. Spectral MAPCA on graphs (Section~\ref{sec:graphs}) admits the $L^{\beta}$ construction flagged at the close of that subsection, with potential applications in spectral clustering and graph signal processing. Deep MAPCA (Section~\ref{sec:deep}) opens the largest research programme: the choice of nonlinearity beyond elementwise activations, the optimisation theory on product generalised Stiefel manifolds, and the empirical evaluation against existing deep architectures are all open. Each of these bridges is, in principle, a separate paper.

A third direction concerns the linearity assumption in Theorem~\ref{thm:uniqueness}. The theorem characterises IPCA among linear data-derived metrics, but the $\beta$-family and other natural MAPCA-style metrics are nonlinear in $\Sigma$. A characterisation of nonlinear equivariant metrics, and of how IPCA sits among them, would extend the structural result of Section~\ref{sec:uniqueness} beyond the Schur-product class. This is a natural follow-up question; whether it admits as clean a uniqueness statement is itself an open problem.

\section{Conclusion}\label{sec:conclusion}

We have developed a precise dictionary positioning Metric-Aware Principal Component Analysis within the geometric deep learning framework. The dictionary is organised across six axes---domain placement, symmetry group, equivariance, invariance, architectural primitive, and geometric prior---and identifies the metric $M$ in MAPCA with the choice of group representation in GDL, with each canonical MAPCA construction corresponding to a specific geometric prior in the GDL sense. The conceptual content of the positioning is anchored by Theorem~\ref{thm:uniqueness}, which characterises IPCA as the unique linear data-derived metric within the MAPCA family that respects $\mathrm{D}_+(p)$-equivariance and projects into the fixed-point set of the action---a characterisation that pins down formally the privileged status of IPCA that the dictionary makes visible.

The primary contribution is the dictionary itself; the uniqueness theorem is the technical anchor. With the dictionary in place, IPCA's scale-invariance property is no longer a coincidence of the diagonal metric but the structural consequence of the GDL design principle applied to the diagonal subgroup of $\mathrm{GL}(p)$; the SSL correspondences of the MAPCA paper become legible as different choices of geometric prior within a common design space; and the natural extensions into kernel methods, spectral graph methods, and deep architectures fall out as the directions in which the linear constant-metric core of MAPCA can be lifted.

Several open directions follow, including the deep MAPCA programme sketched in Section~\ref{sec:deep}, the implications of the geometric-prior reading for tabular-data GDL, and the characterisation of nonlinear equivariant metrics beyond the Schur-product class to which Theorem~\ref{thm:uniqueness} applies. These are taken up in the Discussion as the principal directions in which the programme this paper opens can be continued.


\end{document}